\documentclass{article} 
\usepackage{nips15submit_e,times}

\usepackage{amsmath}
\usepackage{amsthm}
\usepackage{graphicx}
\usepackage{eqlist}
\usepackage{breqn}
\usepackage{todonotes}
\usepackage{subfig}
\usepackage{threeparttable}
\usepackage{scalerel}
\newcommand{\widesim}[2][1.5]{
  \mathrel{\overset{#2}{\scalebox{#1}[1]{$\sim$}}}
}
\usepackage{amsfonts}

\title{Performance Bounds for Pairwise Entity Resolution}

\author{
Matt Barnes\\
School of Computer Science\\
Carnegie Mellon University\\
Pittsburgh, PA 15213 \\
\texttt{mbarnes1@cs.cmu.edu} \\
\And
Kyle Miller \\
School of Computer Science\\
Carnegie Mellon University\\
Pittsburgh, PA 15213 \\
\texttt{mille856@andrew.cmu.edu} \\
\AND
Artur Dubrawski \\
School of Computer Science\\
Carnegie Mellon University\\
Pittsburgh, PA 15213 \\
\texttt{awd@cs.cmu.edu} \\
}

\nipsfinalcopy 

\begin{document}

\maketitle

\begin{abstract}
One significant challenge to scaling entity resolution algorithms to massive datasets is understanding how performance changes after moving beyond the realm of small, manually labeled reference datasets. Unlike traditional machine learning tasks, when an entity resolution algorithm performs well on small hold-out datasets, there is no guarantee this performance holds on larger hold-out datasets. We prove simple bounding properties between the performance of a match function on a small validation set and the performance of a pairwise entity resolution algorithm on arbitrarily sized datasets. Thus, our approach enables optimization of pairwise entity resolution algorithms for large datasets, using a small set of labeled data. 
\end{abstract}

\section{Introduction}
Entity resolution (ER) is the task of identifying records belonging to the same entity (e.g.\ individual, product) across one or multiple datasets. Ironically, it has multiple names: deduplication and record linkage, among others \cite{Getoor2012}. For example, ER is used to disambiguate shopping products \cite{Menestrina2010}, merge datasets of users from disparate sources, or even profile potential terrorist threats. With the use of blocking techniques, entity resolution can be scaled to many millions of records \cite{Papadakis2013}.


The canonical example in Table \ref{tab:canonical} illustrates the usefulness of pairwise ER for these application domains. Initially, the match function may only predict $r_1 \approx r_2$ using the common phone number, where $\approx$ denotes a match. A partial name may not be a strong enough commonality to predict  either of these individually match $r_3$. However, the merge of these records $\left<r_1, r_2\right>$, where $\left< \; \; \right>$ denotes a merge, provides the full name `John Doe' and enables correctly merging all three records.

To design an effective entity resolution system, one would optimize over the ER merge and match functions. 
One might be tempted to evaluate and optimize an ER system on a small dataset with known labels, and then extend this to real-world applications. We stress that performance on small datasets does not necessarily imply similar performance on large datasets. Unlike more traditional machine learning tasks, in ER applications the number of entities often scales linearly with the size of the dataset \cite{Getoor2012}. This is not true in other clustering problems, where the number of clusters is typically constant or sublinear with the dataset size -- a significantly easier problem.
Further, the `no negative evidence' assumption \cite{Getoor2012, Benjelloun2009} can cause a `snowball effect,' wherein several false positives trigger many more clusters to merge, leading to a detrimental degradation in performance. 

Consider the simple example in Figure \ref{fig:synthetic-degredation}, using synthetic data described in Section~\ref{sec:datasets}. First, we learned a match function using a small training dataset of 100 records. On a test dataset of comparable size, it achieved near perfect pairwise precision and recall. However, as we added new entities to the test dataset, pairwise precision significantly degraded -- an extreme example of the entire dataset snowballing into a single entity. More importantly, near perfect performance on the larger datasets \emph{was} possible (dotted line), just with different match function parameters. Using our approach to instead optimize over the larger dataset's estimated lower bound dramatically improves performance on the large set (solid line). 

\begin{table}
\centering
\caption{Canonical Entity Resolution Example}
\label{tab:canonical}
\begin{tabular}{|c |l |l |l |} \hline
Record&Name1&Name2&Phone\\ \hline
$r_1$ & John & D. & 377-8328\\ \hline
$r_2$ & J. & Doe & 377-8328\\ \hline
$r_3$ & John & Doe &  \\
\hline\end{tabular}
\end{table}

\begin{figure}[t]
\begin{center}
        \includegraphics[width=0.5\linewidth]{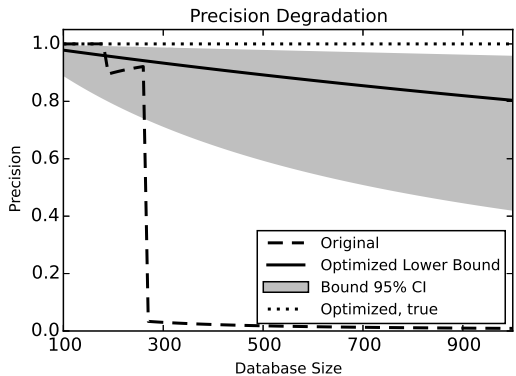}
 \end{center}
        \caption{A simple experiment demonstrates the potential degradation of pairwise precision as the size of the dataset increases. Here the `Original' algorithm (dashed line) was tuned for optimal performance on a training set of 100 records. `Optimized Lower Bound' (solid line) shows our results after instead optimizing model parameters over the larger dataset's estimated lower bound. `True' (dotted line) shows the actual performance corresponding to this lower bound.}
        \label{fig:synthetic-degredation}
\end{figure}

Although performance on a small labeled dataset does not directly equate to performance on an actual larger dataset, some useful information does exist which we will leverage into an estimated lower bound for ER performance on arbitrarily sized problems. Then, optimization of the estimated lower bound allows tuning of pairwise ER systems for large datasets. 

In this paper, our contributions are:
\begin{enumerate}
\item{Theoretical Performance Bounds:} 
We prove simple, estimated, lower bounds on pairwise recall, precision, and $F_1$ performance metrics for arbitrarily sized datasets, under reasonable assumptions and given a small number of labeled record pairs.
\item{Empirical Tightness:} 
We evaluate the bounds on one synthetic and three real world datasets to demonstrate the theoretical bounds are tight to the true performances.
\item{Optimal Merge Function:} 
Given any match function, we prove a lower-bound optimal merge function and `wrapper' for the match function. This conservative strategy is equivalent to finding all connected components, a key insight of the simple bounds.
\end{enumerate}

The remainder of the paper is organized as follows. We begin section \ref{sec:related} with a quick overview of related work in the field of entity resolution. In sections \ref{sec:bounds} and \ref{sec:optimal}, we derive the estimated lower bounds and optimal merge function, respectively. Lastly, in section \ref{sec:results} we demonstrate the empirical tightness of the bound on real world datasets.

\section{Related Work} \label{sec:related}
Entity resolution encompasses a broad set of approaches, including many adapted from the machine learning, optimization, and graph theory domains. Strategies appropriate for ER includes hierarchical clustering \cite{Bilenko2005}, integer linear programming \cite{Ailon2008}, latent Dirichlet allocation \cite{Bhattacharya2007}, pairwise match/merge \cite{Benjelloun2009}, Markov logic \cite{Singla2006} and hybrid human-machine systems \cite{Wang2012}. Pairwise entity resolution approaches are appealing because they use an intuitive and easy to implement iterative match and merge process between pairs of records. Further, under certain assumptions, pairwise algorithms will perform the optimal number of record comparisons \cite{Benjelloun2009}.

Perhaps the most general framework for pairwise entity resolution was presented by Benjelloun et al. \cite{Benjelloun2009}. They outlined a theoretically disciplined approach, wherein certain properties of the match and merge function guarantee a deterministic output in the optimal number of record comparisons. We explore the use of some of these properties in the derivation of our bounds. Collectively, these properties are referred to by their acronym ICAR:

\begin{enumerate}
\item{Idempotence:} 
$\forall r, r \approx r$ and $\left<r,r\right> = r$.
\item{Commutativity:} 
$\forall r_1, r_2, r_1 \approx r_2$ iff $r_2 \approx r_1$, and if $r_1 \approx r_2$, then $\left< r_1, r_2 \right> = \left< r_2, r_1 \right>$.
\item{Associativity:} $\forall r_1, r_2, r_3$ such that $\left<r_1, \left<r_2, r_3\right>\right>$ and $\left<\left<r_1, r_2\right>, r_3\right>$ exist, $\left<r_1, \left< r_2, r_3 \right>\right> = \left<\left<r_1, r_2\right>, r_3\right>$.
\item{Representativity:} If $r_3 = \left<r_1, r_2\right>$ then for any $r_4$ such that $r_1 \approx r_4$, we also have $r_3 \approx r_4$.
\end{enumerate}

The first three properties are straightforward and reasonable to assume for most ER systems. The crux of determinism falls on the final property, representativity. We, too, will take advantage of this convenient property, leaving the interesting problem of how relaxing this assumption affects the performance bounds for future work. Intuitively, representativity means merging any two records can only monotonically increase their chance of matching with other records. This is also referred to as the `no negative evidence' clause. 


\section{Lower Bounds of Performance} \label{sec:bounds}
Although many metrics exist to evaluate entity resolution performance when a ground truth dataset is available, this is rarely the case. Not surprisingly, human-generated clusterings rarely number beyond a thousand records \cite{Menestrina2010} -- a relatively easy ER problem. Even finding publicly available datasets with ground truth so that we could objectively evaluate our results was a trying task.

In the simplest setting, we assume we have access to some pairs with known binary match/mismatch label $y$, such that $x \scalerel*{\widesim{iid}}{j}p(x|y)$. 
For large datasets, finding all records belonging to one entity is a worst-case combinatorial problem, but finding just two matching records is relatively easy using a hybrid human-machine system \cite{Wang2012} or with strong features (e.g.\ phone number, product ID)

With both match and mismatch pairs at our disposal, we created a training and validation set of labeled pairs. The remaining records form the test dataset. Note the training and validation sets will likely have significantly different class balance, cluster sizes, and overall number of samples than the test set. Though an entity resolution algorithm may perform well on the validation set with few samples and small cluster sizes, this may not indicate strong performance on the full dataset with millions of records and many more clusters. In practice, a developer needs to know performance guarantees of the test set because this is the deployed system. 

Here, we derive precise relationships between the performance of the match function on the validation record pairs and estimated lower bounds on ER pairwise precision, recall, and $F_1$ on the test set. Our notation for the following proofs, which the reader may find convenient to refer back to, is:

\begin{eqlist}

\item[$\left<r_i, r_j \right>$]
Record formed by merging records $r_i$ and $r_j$.

\item[$V$]
Set of validation record pairs with known labels, $V = \{(r_1, r_1'), \dotsc, (r_m, r_m')\}$.

\item[$V_S$]
Set of record pairs in the validation set with positive label, \\$V_S = \{(r_i, r_i'): y_i=1$, $\forall (r_i, r_i') \in V\}$.

\item[$V_M$]
Set of record pairs in the validation set that are predicted to directly match, \\$V_M = \{(r_i, r_i'): r_i \approx r_i'$, $\forall (r_i, r_i') \in V\}$.

\item[$T$]
Set of test records $\{r_1, \dotsc, r_n\}$.

\item[$T_M$]
Set of record pairs in the test set that are predicted to directly match, \\$T_M=\{\left(r_i,r_j\right): r_i \approx r_j, i<j$, $\forall r_i, r_j \in T\}$.

\item[$R$]
Set of record pairs in the entity resolution clustering of the test set.

\item[$S$]
Set of record pairs in the true clustering of the test set (unknown).

\item[$Prec(R,S)$]
Precision of predicted and true positive pairs, $Prec(R,S) = |R \cap S|/|R|$.

\item[$Recall(R,S)$]
Recall of predicted and true positive pairs, $Recall(R,S) = |R \cap S|/|S|$.

%

\item[$C_{V}$]
Class balance of pairs in the validation set, $C_{V} = |V_S|/|V|$.

\item[$C_{T}$]
Estimated class balance of pairs in the test set, $C_{T} = |S|/|Pairs(T)|$.

\end{eqlist}

\newtheorem{lemma}{Lemma}
\begin{lemma} For entity resolution systems satisfying the representativity property, every record pair that directly matches will end up in the same entity.

\begin{equation}
T_M \in R.
\end{equation}
\end{lemma}

Additional pairs in $R$ can occur from chains of matches (i.e.\ $r_1 \approx r_2$, $r_2 \approx r_3$, thus $(r_1, r_3) \in R$) and from merging (see Table \ref{tab:canonical}). However, we are unable to make strong claims about the additional matches since composite records do not occur in the validation set.

\begin{proof}
Suppose on the contrary there exists a pair of records $(r_1, r_2)$, such that $(r_1, r_2) \in T_M$ but  $(r_1, r_2) \not\in R$. In other words, $r_1 \approx r_2$ and they are resolved to separate entities $I_1 = \left<r_1, ....\right>$ and $I_2 =\left<r_2, ....\right>$. Since these clusters were not merged in the ER process, $\left<r_1, ....\right> \not\approx \left<r_2, ....\right>$, which contradicts the representativity property.
\end{proof}

\newtheorem{theorem}{Theorem}
\begin{theorem} \label{theorem:precision}
The pairwise precision of an entity resolution result can be lower bounded by:
\begin{equation}
\mathbb{E}\left[Prec(R,S)\right] \geq \frac{|T_M|}{|R|}\left(\frac{C_T(1-C_V)\mathbb{E}\left[Prec(V_M, V_S)\right]}{C_V(1-C_T) + (C_T-C_V)\mathbb{E}\left[Prec(V_M, V_S)\right]}\right).
\label{eq:prec-lower-bound}
\end{equation}
\end{theorem}
The bound is composed of two parts. $|T_M|/|R|$ is the fraction of record pairs in the test set entity resolution that directly match, which we can make stronger claims about. $Prec(V_M, V_S)$ is the precision of these direct matches, adjusted for the change in class balance.
\begin{proof}
From Lemma 1 and applying the definitions of pairwise precision for $R$ and $T_M$:
\begin{align}
\mathbb{E}\left[Prec(R, S)\right] &= \mathbb{E}\left[\frac{|R \cap S|}{|R|}\right], \nonumber \\
&\geq \mathbb{E}\left[\frac{|T_M \cap S|}{|R|}\right], \nonumber \\
&= \frac{|T_M|}{|R|}\mathbb{E}\left[Prec(T_M, S)\right], \nonumber \\
&\geq \frac{|T_M|}{|R|}\left(\frac{C_T(1-C_V)\mathbb{E}\left[Prec(V_M, V_S)\right]}{C_V(1-C_T) + (C_T-C_V)\mathbb{E}\left[Prec(V_M, V_S)\right]}\right), \nonumber
\end{align}
where the last step follows from equating the match function validation set performance to the expected match function test set performance using change in match/mismatch class balance.
\end{proof}

Most of the values are straightforward to count from the resolution. $|R|$ is the number of pairs in the clustering output. $|T_M|$ is the number of records that directly match, which by Lemma 1 can be efficiently computed as $\sum_{(r_1, r_2) \in R} r_1 \approx r_2$.
 
The class balance of the validation set $C_{V}$ is known, but we must estimate $C_{T}$. We refer the reader to state-of-the-art results for class prior estimation~\cite{duPlessis2014, Saerens2002}. 

\begin{theorem} \label{theorem:recall}
The pairwise recall of an entity resolution result can be lower bounded by:
\begin{equation}
\mathbb{E}\left[Recall(R,S)\right] \geq \mathbb{E}\left[Recall(V_M, V_S)\right]. \label{eq:recall-lower-bound}
\end{equation}
\end{theorem}
In other words, the recall on the validation set already forms a lower bound for the pairwise recall on the test resolution.
\begin{proof}
From the definitions of pairwise recall for $T_M$ and $R$ and then applying Lemma 1:
\begin{align}
\mathbb{E}\left[Recall(R,S)\right] &= \mathbb{E}\left[\frac{|R \cap S|}{|S|}\right], \nonumber \\
&\geq \mathbb{E}\left[\frac{|T_M \cap S|}{|S|}\right], \nonumber \\
&= \mathbb{E}\left[Recall(T_M, S)\right], \nonumber \\
&= \mathbb{E}\left[Recall(V_M, V_S)\right], \nonumber
\end{align}
where the last step does not require class rebalancing because recall is not a function of class balance (unlike precision, it only depends on the positive pairs).
\end{proof}

A lower bound on pairwise $F_1$ (the harmonic mean of pairwise precision and recall) can be computed with the two former lower bounds. We will focus more on measuring both pairwise precision and recall as they are more informative than the aggregated $F_1$ metric.

\section{Optimal Merge Function}\label{sec:optimal}
Given any match function $m$ satisfying the idempotence and commutativity properties, we will prove a merge function and `wrapper' match function that optimize the estimated lower bounds. Since the idempotence property is trivially satisfied for any match function by checking for identical records and the commutativity property is satisfied by checking both directions $r_1 \approx r_2$ and $r_2 \approx r_1$, this essentially holds for all pairwise match functions. These match and merge functions form a conservative strategy, but provide the lower bound optimal performance given only labeled pairs.

We consider the original set of records $R = \{r_1, \dotsc, r_n\}$ and use notation $o$ for a record formed by merging at least two other records.

\begin{theorem} \label{theorem:optimal}
For any match function, the pairwise precision, recall, and $F_1$ estimated lower bounds are optimal for the merge function:
\begin{equation}
\left<o_1, o_2\right> = \bigcup_{r_i \in o_1, o_2} r_i
\end{equation}
The corresponding `wrapper' match function between $o_1$ and $o_2$ is:
\begin{equation}
o_1 \approx o_2 = \max_{r_i \in o_1, r_j \in o_2} m(r_i, r_j).
\end{equation}
\end{theorem}

\begin{proof}
We will show both directions, that the optimal merge function and match `wrapper' must make at least these matches to satisfy the ICAR properties, and that any additional matches will decrease the estimated performance lower bound. By the definition of the set union operator, the merge function is associative. The rest of the proof will focus on the representativity property.


\emph{Direction 1}: We are constrained by match and merge functions that satisfy the ICAR properties. In the first direction, we will show these are the minimum matches required to satisfy representativity. Assume on the contrary: there exist two composite records $o_1$ and $o_2$, such that $o_1 \not\approx o_2$ but one pair of their constituent records match, i.e.\ $r_i \approx r_j$, for some $r_i \in o_1$, $r_j \in o_2$. By definition, this contradicts the representativity property.

\emph{Direction 2}: In the second direction, we will show any additional matches will increase $|R|$ and thus decrease the estimated pairwise precision lower bound. Assume there exist two records $o_1$ and $o_2$, such that $o_1 \approx o_2$ but none of their constituent records match, i.e.\ $r_i \not \approx r_j$, $\forall r_i \in o_1$, $r_j \in o_2$. The additional match $o_1 \approx o_2$ may increase $|R|$, thus decreasing $Prec(R, S)$.
\end{proof}

The simplicity of this approach is derived from only claiming performance knowledge of direct record matches from the validation set performance. Interestingly, this ER system is equivalent to finding all connected components, where each edge $A_{ij} = r_i \approx r_j$ in the adjacency matrix $A$. We stress that though this may optimize the estimated lower bound performances, it does not necessarily guarantee better performance. However, if ground truth is not available for a dataset of comparable size to the deployed system, then this is now a theoretically well motivated approach.

A significant benefit of Theorem~\ref{theorem:optimal} is the provided match function need not satisfy the very restrictive representativity property. Further, since the idempotence and commutativity properties are trivial to satisfy, $m$ can be essentially any match function. For example, one could use more complex machine learning based match functions (e.g.\ kernelized SVM, random forests) and featurizations which may not have intuitive merge operations (e.g.\ word2vec \cite{Mikolov2013}, Brown clustering \cite{Brown1992}). Using less restrictive match functions undoubtedly enables better $Prec(V_M, V_S)$ and $Recall(V_M, V_S)$, further improving the lower bounds.

\section{Experiments} \label{sec:results}

We conducted experiments on multiple datasets with known ground truth to empirically demonstrate the tightness of the estimated lower bounds. Specifically, we are interested in optimizing ER model parameters over the estimated lower bounds and over the ground truth metrics to show they achieve similar results. 

\subsection{Datasets} \label{sec:datasets}

We used one synthetic and three real world datasets with known ground truth for our experiments, as described in Table \ref{tab:datasets}. For all these datasets, the goal of entity resolution is to find records describing the same entity (e.g.\ restaurant, product, or person). For the synthetic dataset, we generated each record's features using a feature vector unique to its respective entity, plus random Gaussian noise. 

\begin{table}[b]
\centering
\begin{threeparttable}[b]
\caption{Datasets used in the experiments}
\label{tab:datasets}
\begin{tabular}{l r r r} \hline
Dataset&\# dim&\# records & \# matches \\ \hline
Synthetic & 10 & 1000 & 4500\\
Restaurant\tnote{1} & 4 & 864 & 112 \\
Abt-Buy\tnote{2} & 3 & 2173 & 1118 \\
Escort (subset) & 20 & 10000 & 10596 \\
\hline\end{tabular}
\footnotesize{
\begin{tablenotes}
\item[1] http://www.cs.utexas.edu/users/ml/riddle/data/restaurant.tar.gz
\item[2] http://dbs.uni-leipzig.de/file/Abt-Buy.zip
\end{tablenotes}}
\end{threeparttable}
\end{table}

\begin{figure*}[p]
\centering
\subfloat[Synthetic precision]{
  \includegraphics[width=0.44\linewidth]{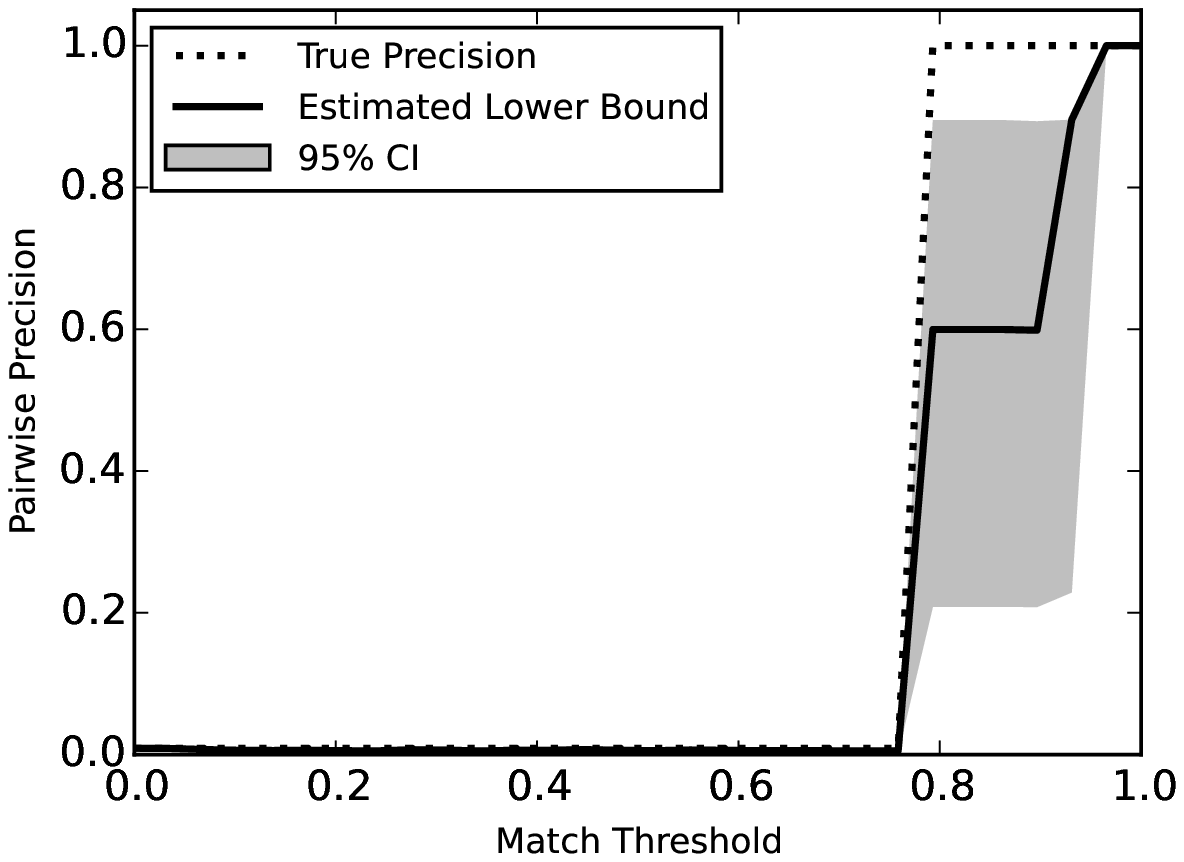}
}
\subfloat[Synthetic recall]{
  \includegraphics[width=0.44\linewidth]{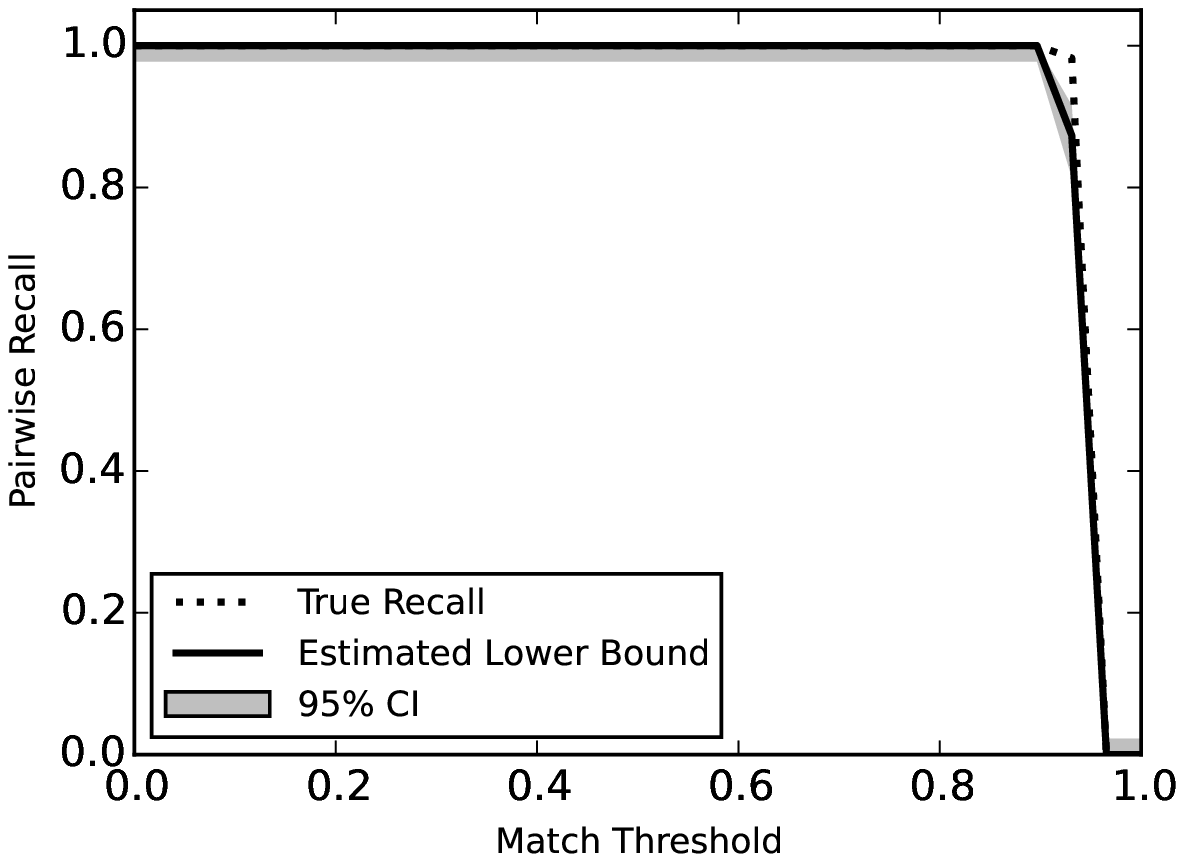}
}
\vspace{-3mm}
\subfloat[Restaurant precision]{
  \includegraphics[width=0.44\linewidth]{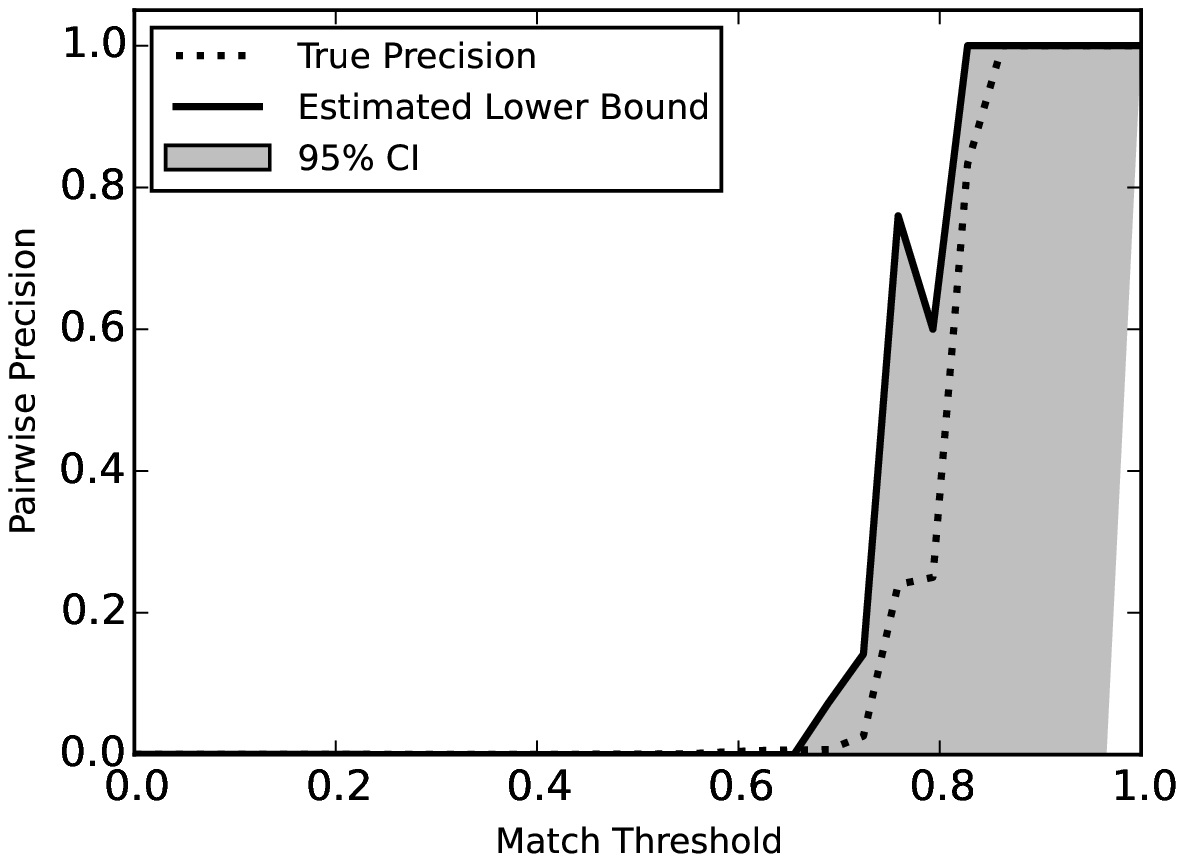}
}
\subfloat[Restaurant recall]{
  \includegraphics[width=0.44\linewidth]{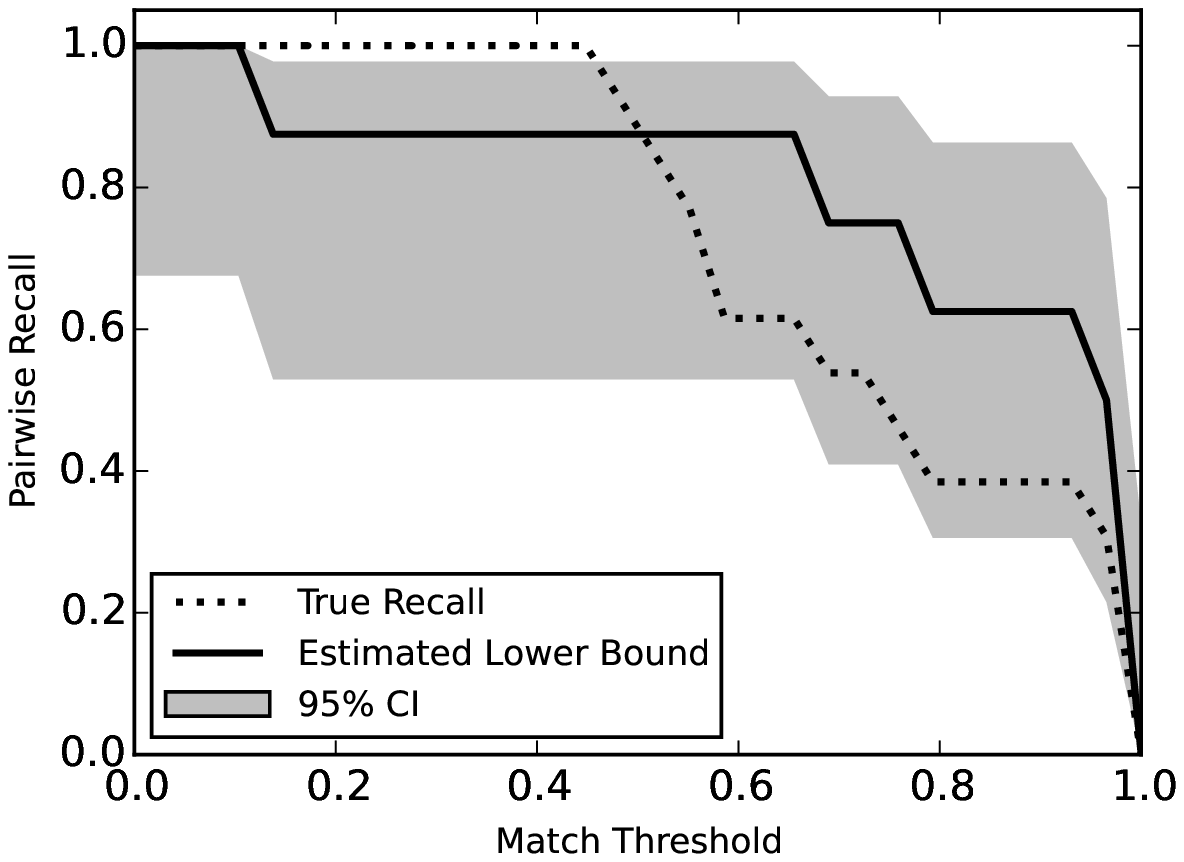}
}
\vspace{-3mm}
\subfloat[Abt-Buy precision]{   
  \includegraphics[width=0.44\linewidth]{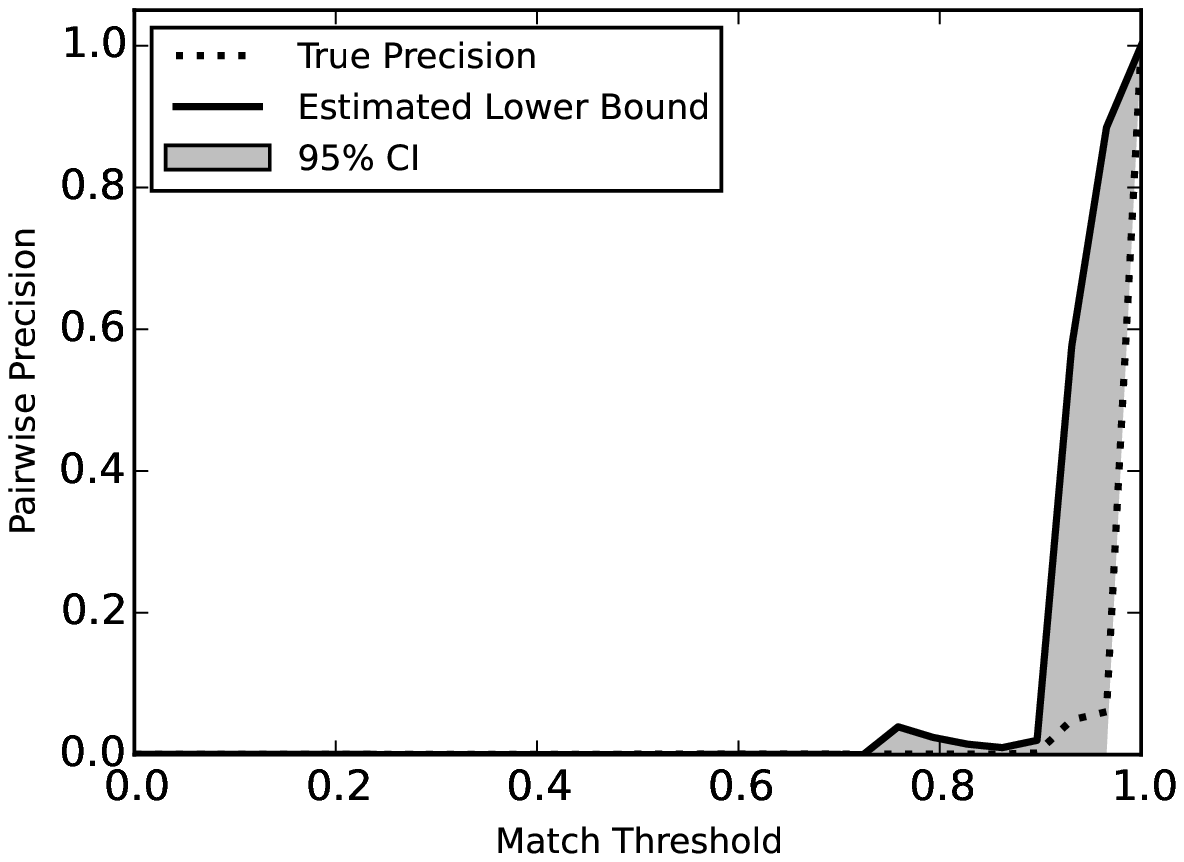}
}
\subfloat[Abt-Buy recall]{
  \includegraphics[width=0.44\linewidth]{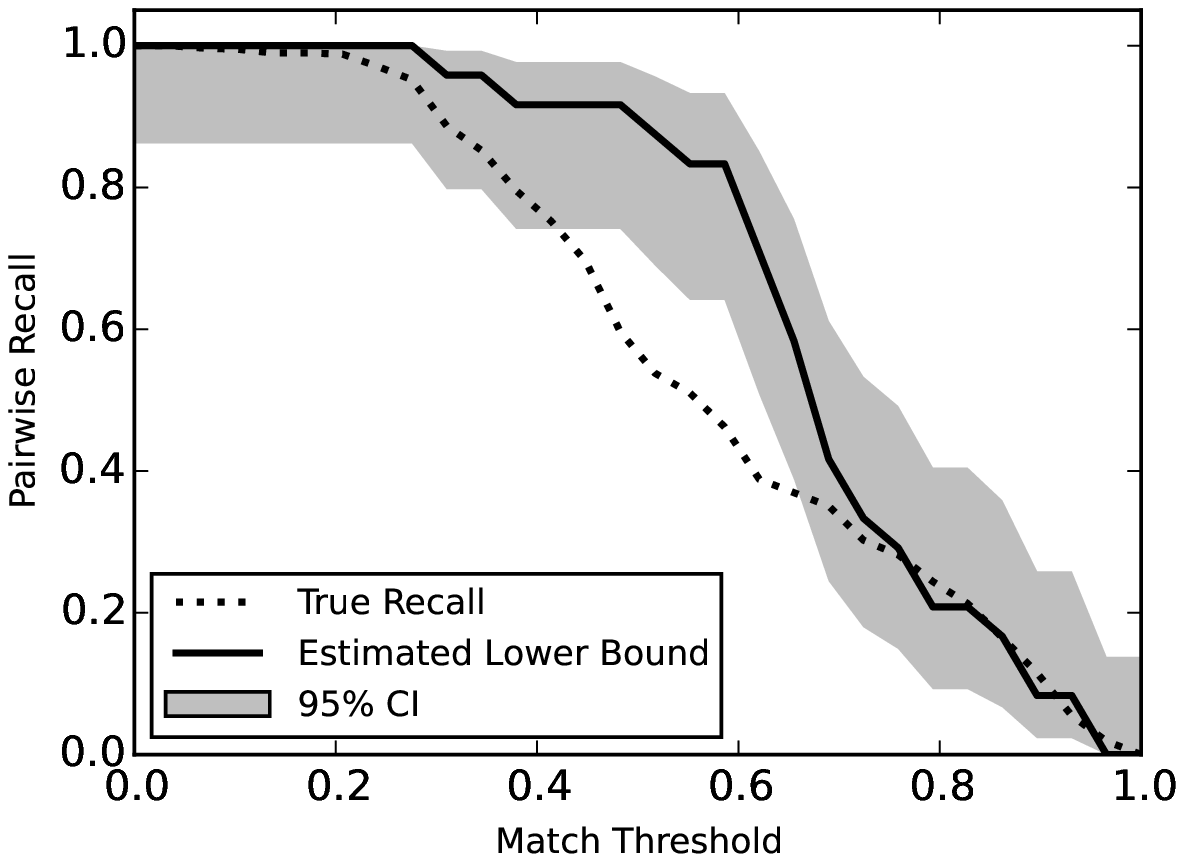}
}
\vspace{-3mm}
\subfloat[Escort precision]{   
  \includegraphics[width=0.44\linewidth]{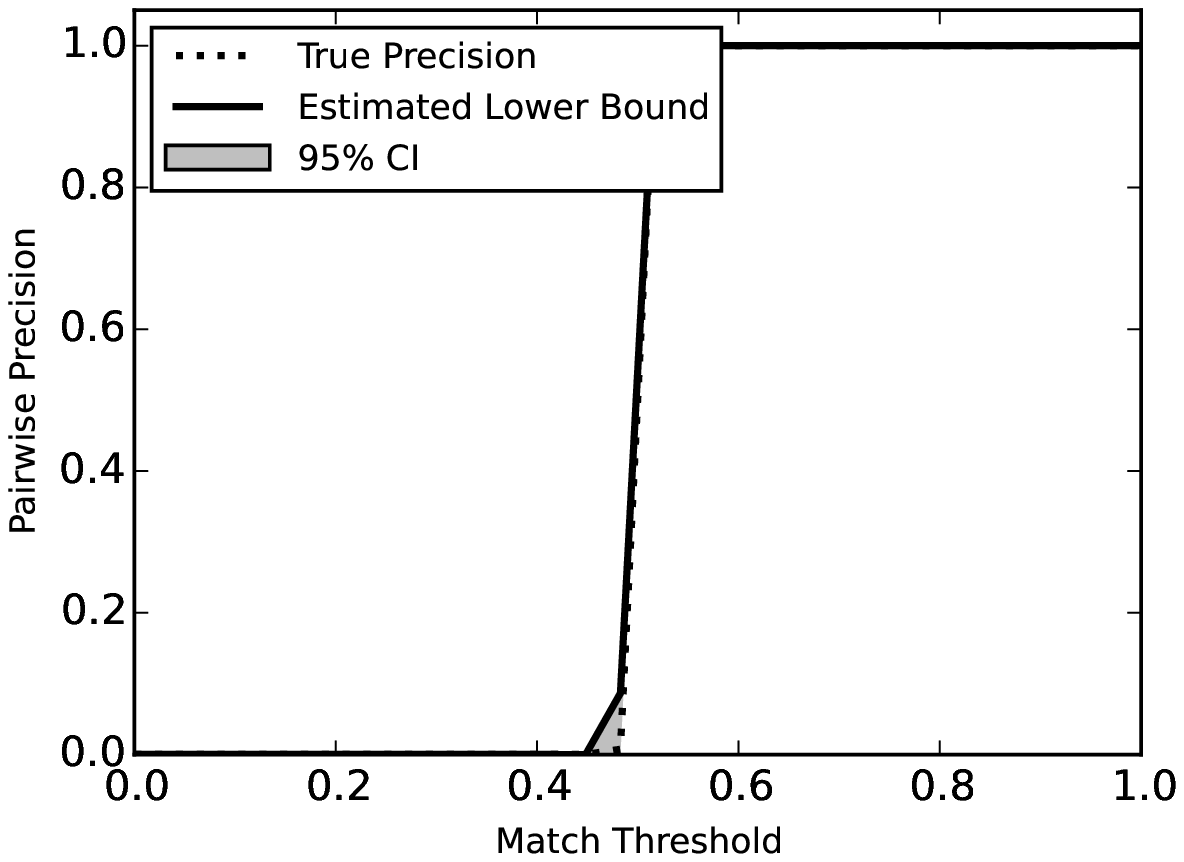}
}
\subfloat[Escort recall]{
  \includegraphics[width=0.44\linewidth]{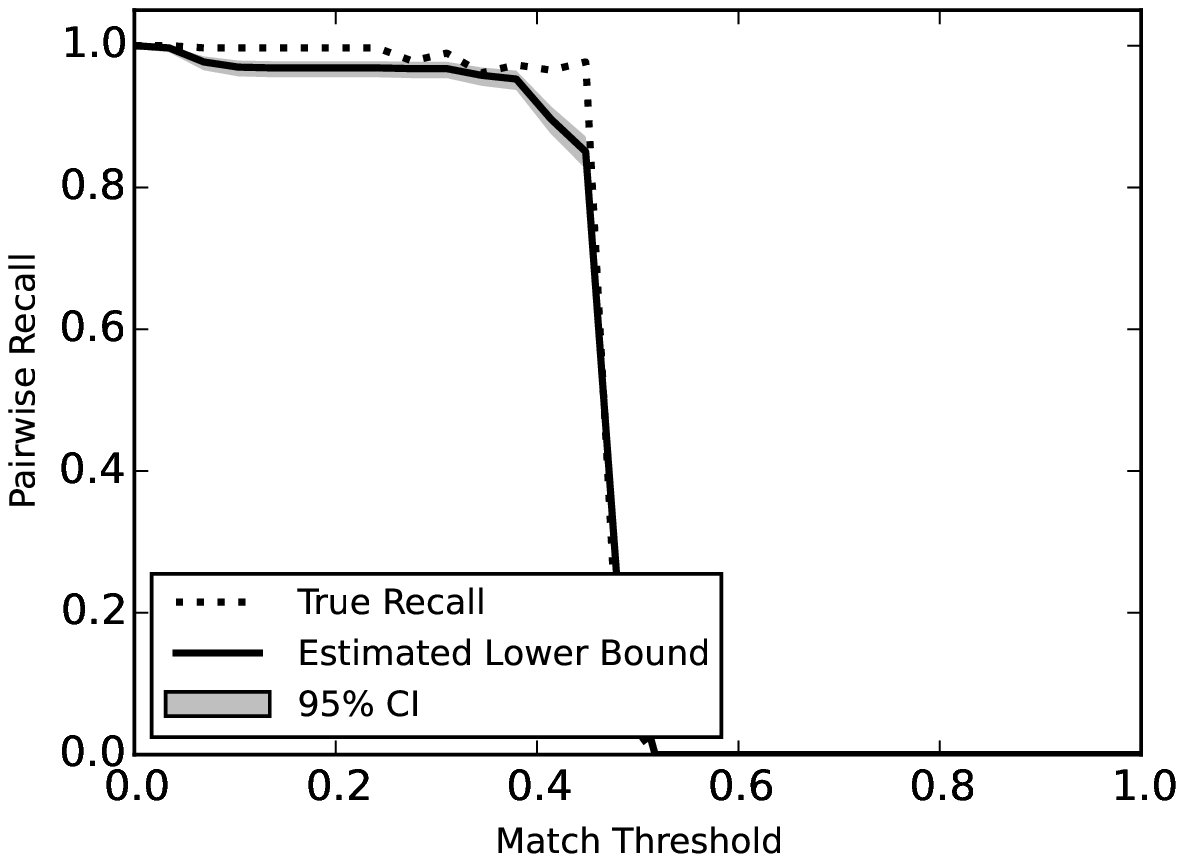}
 }
\caption{Experimental results demonstrate model parameters can be tuned to optimize estimated lower bound pairwise precision and recall of the test set. The resulting estimated lower bound is close to the true performance. Pairwise $F_1$ is not shown because it is the harmonic mean of the two former metrics, and is thus less informative. 
}
\label{fig:results1}
\end{figure*}

Unlike general machine learning tasks, publicly available entity resolution datasets with known ground truth are extremely limited, and do not number beyond several thousand records. The restaurant dataset is one of the earliest ER tasks discussed in literature \cite{Tejada2001}, and still used today \cite{Wang2012, Kopcke2008}. 
Unfortunately, the dataset is also relatively small -- numbering only 864 records and five features (name, phone number, street address, city, cuisine). We threw away the phone number feature because it made the problem too simple. The Abt-Buy dataset is more recent, larger at 2173 records, and used extensively in current research \cite{Wang2012, Kopcke2010}.  It consists of product information from two retailers, including product name, description, and price.

Both the Restaurant and Abt-Buy datasets are a class of entity resolution known as clean-clean, wherein two `clean' datasets with completely resolved entities are merged together \cite{Papadakis2013}. This problem is easier than the more general problem of resolving entities with an unknown number of records. To formulate these datasets in a more general context, we merged them together into a single `dirty' dataset and ignored the advantageous `clean-clean' knowledge in our experiments.

Lastly, we evaluated a subset of a personal ads dataset scraped from escort advertising websites over the past few years \cite{Greenemeier2015}. 
We used natural-language-processing algorithms to extract 20 features, such as name, age, location, and hair color of the person being advertised. For ground truth, we used a subset of the data containing phone number matches as a proxy label. Although phone numbers will not allow us to discover the full ground truth, it is reasonable to assume ads with the same phone number belong to the same entity (i.e. person or group) because those numbers are the means of contact for potential customers.


\subsection{Entity Resolution}
We used the R-Swoosh algorithm for our ER systems \cite{Benjelloun2009}. For the merge function, we simply used the set union of the respective features. For example, in Table \ref{tab:canonical}, $\left<r_1, r_2\right>$ would be [\{J., John\}, \{D., Doe\}, \{377-8328\}].

For the match function, we trained a binary logistic regression classifier using known matches and mismatches in the training dataset. Like all pairwise entity resolution algorithms, it operates on pairwise features, which we computed from two records' features using either a binary match (e.g.\ state, hair color), numerical difference (e.g.\ ages, weights), or Levenshtein string edit distance (e.g.\ name) of each feature pair. If a record had multiple of a particular feature from a merge operation, we used the closest feature match. %

Another benefit of using a probabilistic match function is the choice of parameters is reduced to a single value: the cut-off threshold. 
The choice of cut-off threshold is a classic trade-off between precision and recall -- an ideal setting to examine the results of our bounds.

\subsection{Results}
To examine the efficacy of the estimated lower bound in tuning an entity resolution system, we evaluated the true and lower bound performances across tightly spaced intervals of match cut-off thresholds, as shown in Figures \ref{fig:results1}. The tightness of the bounds demonstrate two important qualities. First, they enable the optimization of model parameters (e.g.\ cut-off threshold) using the estimated lower bound. Though this may not necessarily result in the true (unknown) optimal parameters, it will result in the best estimated lower bound. Second, it enables enforcing a level of acceptable quality prior to the use of any entity resolution results. 

One may be surprised to see the estimated lower bound exceed the true performance. This is, indeed, possible because of uncertainty in estimations of $Prec(V_M, V_S)$, $Recall(V_M, V_S)$ and $C_T$. The 95\% confidence intervals are obtained via the propagation of validation set Wilson scores for precision and recall \cite{Wilson1927}. Uncertainty increases as the gap between validation set and test set sizes widen, a phenomenon observable in Figure \ref{fig:synthetic-degredation}. For very small datasets such as Restaurant, we were restricted to using minimal validation samples due to the small number of labels. However, for larger experiments such as Abt-Buy and Escort, we could afford hundreds or thousands of validation samples, significantly reducing uncertainty. This is also theoretically motivated by the shift in class balance in Theorem 1. 

The four experiments demonstrate different ER behavior. The synthetic experiment has a narrow range of model parameters with perfect precision and recall, where performance degrades dramatically outside this range. The Restaurant experiment has a more gradual tradeoff between precision and recall, though there is a significant uncertainty in the lower bound estimate due to the limited number of validation samples. Precision in Abt-Buy quickly degrades, though recall is much more gradual. 
Our bounds correctly capture the need to improve the underlying ER systems for the Abt-Buy and Escort datasets. Without this lower bound, the poor performance on larger datasets would not be evident from smaller tests.

\section{Conclusions}
Performance optimization of scalable entity resolution systems is challenging because unlike other machine learning tasks, there is not a clear understanding of how behavior will change on larger datasets. In this paper, we developed a simple -- yet effective -- method for optimizing lower bound performance using a small set of labeled pairs.

Further, we showed the optimal lower bound strategy for any match function is the connected components problem from graph theory -- a relatively conservative clustering approach compared to many ER systems. We understand that this does not necessarily guarantee better performance, but it does provide a better lower-bound guarantee. For instance, in our original example in Table \ref{tab:canonical}, $r_3$ would have matched neither $r_1$ nor $r_2$. However, when labeled datasets of comparable size to the deployed system are not available, this is now a theoretically well motivated approach.

Our bounds specifically addressed performance of pairwise entity resolution algorithms satisfying the ICAR properties \cite{Benjelloun2009}. Pairwise algorithms are intuitive, easy to implement, and perform an optimal number of pairwise record comparisons. However, they are also only a subset of entity resolution approaches \cite{Getoor2012, Bilenko2005, Ailon2008, Bhattacharya2007, Singla2006, Wang2012}. Further, we only considered pairwise precision, recall, and $F_1$ due to their popularity, intuitive interpretation and mathematical convenience, though other existing metrics have been shown to produce conflicting rankings \cite{Menestrina2010}.

Estimating the lower bounds relies on accurate estimations of several other quantities, including recall and precision on the validation set and class prevalence estimation in the test set. Especially as datasets scale to much larger sizes, our bounds rely on these estimates. As evident in Theorem~\ref{theorem:precision} and in our experiments, uncertainty increases as the gap between validation and testing set sizes widened.


%


\bibliographystyle{unsrt}

{\small
\bibliography{barnes_nips}}

\end{document}